\documentclass[sigconf]{acmart}
\usepackage{amsthm}

\newtheorem{theorem}{Theorem}

\setcopyright{none}
\renewcommand\footnotetextcopyrightpermission[1]{}
\renewcommand{\shortauthors}{Geng et al.}
\AtBeginDocument{%
  }

\usepackage{graphicx}
\usepackage{booktabs}
\usepackage{multirow}
\usepackage[table]{xcolor}
\usepackage{tabularx}
\usepackage{array}
\usepackage{threeparttable}
\usepackage{adjustbox}
\usepackage{makecell}
\usepackage{subcaption}

\usepackage{algorithm}
\usepackage{algorithmic}

\usepackage{enumitem}
\usepackage{pifont}

\usepackage{amsmath}
\usepackage{amsthm}
\usepackage{mathtools}
\usepackage{bm}

\newcolumntype{Y}{>{\raggedright\arraybackslash}X}
\newcolumntype{L}[1]{>{\raggedright\arraybackslash}m{#1}}
\newcolumntype{C}[1]{>{\centering\arraybackslash}m{#1}}

\setcopyright{acmlicensed}
\copyrightyear{2018}
\acmYear{2018}
\begin{document}

\title{Seeing is Believing: Robust Vision-Guided Cross-Modal Prompt Learning under Label Noise}


\settopmatter{authorsperrow=3}

\author{Zibin Geng}
\authornote{Both authors contributed equally to this research.}
\affiliation{%
  \institution{Institute of Computing Technology, Chinese Academy of Sciences}
  \institution{University of Chinese Academy of Sciences}
  \city{Beijing}
  \country{China}}
\email{gengzibin25z@ict.ac.cn}

\author{Xuefeng Jiang}
\authornotemark[1]
\affiliation{%
  \institution{Institute of Computing Technology, Chinese Academy of Sciences}
  \institution{University of Chinese Academy of Sciences}
  \city{Beijing}
  \country{China}}
\email{jiangxuefeng21b@ict.ac.cn}

\author{Jia Li}
\affiliation{%
  \institution{Institute of Information Engineering, Chinese Academy of Sciences}
  \institution{University of Chinese Academy of Sciences}
  \city{Beijing}
  \country{China}}
\email{lijia3000@foxmail.com}

\author{Zheng Li}
\affiliation{%
  \institution{PCALab, VCIP, College of Computer Science, Nankai University}
  \city{Tianjin}
  \country{China}}
\email{zhengli97@mail.nankai.edu.cn}

\author{Tian Wen}
\affiliation{%
  \institution{Institute of Computing Technology, Chinese Academy of Sciences}
  \institution{University of Chinese Academy of Sciences}
  \city{Beijing}
  \country{China}}
\email{wentian24s@ict.ac.cn}

\author{Lvhua Wu}
\affiliation{%
  \institution{Institute of Computing Technology, Chinese Academy of Sciences}
  \institution{University of Chinese Academy of Sciences}
  \city{Beijing}
  \country{China}}
\email{wulvhua24s@ict.ac.cn}

\author{Sheng Sun}
\affiliation{%
  \institution{Institute of Computing Technology, Chinese Academy of Sciences}
  \city{Beijing}
  \country{China}}
\email{sunsheng@ict.ac.cn}

\author{Yuwei Wang}
\affiliation{%
  \institution{Institute of Computing Technology, Chinese Academy of Sciences}
  \city{Beijing}
  \country{China}}
\email{ywwang@ict.ac.cn}

\author{Min Liu}
\authornote{Corresponding author.}
\affiliation{%
  \institution{Institute of Computing Technology, Chinese Academy of Sciences}
  \city{Beijing}
  \country{China}}
\email{liumin@ict.ac.cn}
\renewcommand{\shortauthors}{Trovato et al.}

\begin{abstract}

Prompt learning is a parameter-efficient approach for vision-language models, yet its robustness under label noise is less investigated. Visual content contains richer and more reliable semantic information, which remains more robust under label noise. However, the prompt itself is highly susceptible to label noise. Motivated by this intuition, we propose VisPrompt, a lightweight and robust vision-guided prompt learning framework for noisy-label settings. Specifically, we exploit a cross-modal attention mechanism to reversely inject visual semantics into prompt representations. This enables the prompt tokens to selectively aggregate visual information relevant to the current sample, thereby improving robustness by anchoring prompt learning to stable instance-level visual evidence and reducing the influence of noisy supervision. To address the instability caused by using the same way of injecting visual information for all samples, despite differences in the quality of their visual cues, we further introduce a lightweight conditional modulation mechanism to adaptively control the strength of visual information injection, which strikes a more robust balance between text-side semantic priors and image-side instance evidence. The proposed framework effectively suppresses the noise-induced disturbances, reduce instability in prompt updates, and alleviate memorization of mislabeled samples. VisPrompt significantly improves robustness while keeping the pretrained VLM backbone frozen and introducing only a small amount of additional trainable parameters. Extensive experiments under  synthetic and real-world label noise demonstrate that VisPrompt generally outperforms existing baselines on seven benchmark datasets and achieves stronger robustness.
Our code is publicly available at \url{https://github.com/gezbww/Vis_Prompt}.
\end{abstract}




\maketitle

\section{Introduction}
\label{sec:intro}

Vision-language models (VLMs) \cite{CLIP,jia2021scaling,li2025surveystateartlarge} pretrained on large-scale image-text corpora have exhibited remarkable transferability, enabling strong zero-shot performance across a wide range of downstream recognition and retrieval tasks. 
Adapting such pretrained VLMs to a specific task or domain still requires a lightweight interface that is both data-efficient and parameter-efficient.
Prompt learning has therefore become a widely adopted paradigm: It keeps the pretrained backbone frozen and optimizes only a small set of learnable context tokens, often achieving competitive performance with minimal trainable overhead \cite{jia2022vpt,lester-etal-2021-power,li-liang-2021-prefix,DBLP:journals/corr/abs-2110-07602,CoOp,10.1145/3627673.3679529}.

Despite this efficiency, prompt learning remains highly sensitive to annotation quality in practice \cite{CoOp}. Real-world datasets often contain mislabeled samples, and such corrupted supervision can directly mislead the optimization of soft prompts.

This issue is especially pronounced in prompt learning: Since the image and language encoders are frozen, the burden of task adaptation is concentrated on a small number of prompt embeddings, making them highly exposed to noisy gradients induced by incorrect labels.
As the noise rate increases, the learned prompts can gradually drift toward spurious class semantics, leading to unstable optimization and degraded generalization.

\begin{figure}[t]
    \centering
    \includegraphics[width=\columnwidth]{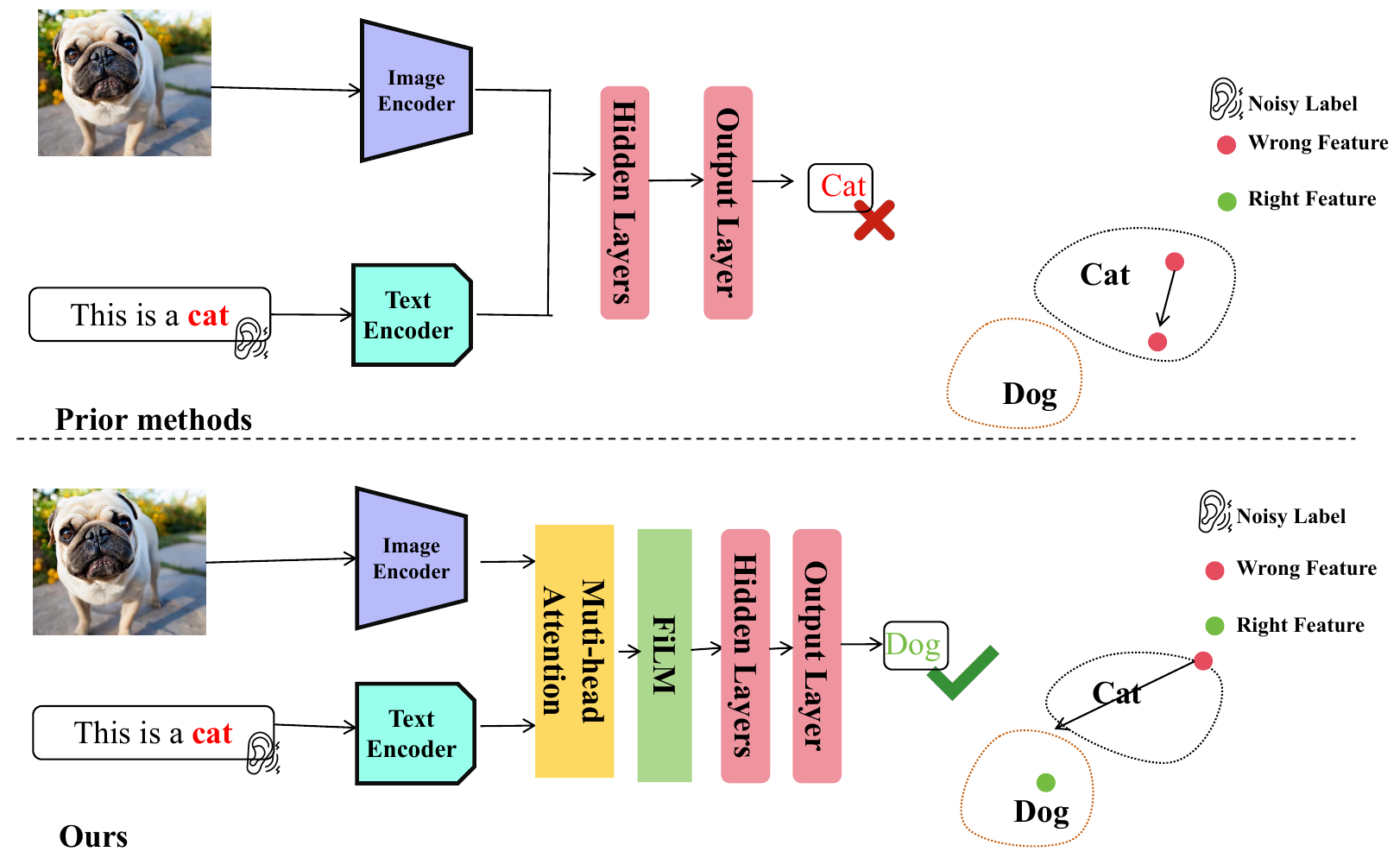}
    \caption{Prior label-driven prompt methods can be misled by noisy supervision and produce incorrect predictions, while our method injects reliable image-grounded information to guide prompt learning toward the correct class.}
    \label{fig:teaser}
\end{figure}

Recent studies have shown that prompt learning with frozen VLM backbones exhibits a certain degree of tolerance to label noise, which has motivated a growing body of research on noise-robust prompt adaptation \cite{wuWhyPromptTuning2023,Guo2024JoAPR}. 
Existing methods \cite{Guo2024JoAPR} mainly improve robustness from three directions, including confidence-based noise detection and trust estimation \cite{wei2024deft,TrustCLIP}, robust objective design \cite{Pan2025NLPrompt,Hu_2025_BMVC}, and sample refinement strategies \cite{Pan2025NLPrompt}. 
These approaches have achieved encouraging results, but most of them still tackle noisy supervision mainly through label-centric mechanisms, rather than rethinking prompt optimization from the perspective of visual guidance.
This design pattern leaves a more reliable source of information underexplored. In contrast, even under corrupted annotations, large-scale pretrained VLMs such as CLIP \cite{CLIP} can still provide rich instance-level semantic cues from images, offering a more stable source of guidance under label noise. However, current methods do not fully leverage such visual cues as a primary anchor for prompt optimization. Instead, they often rely on early-stage confidence estimation or sample discrimination to identify clean data under noisy supervision. Once these early decisions are inaccurate, the resulting confirmation bias can be inherited by subsequent optimization, causing prompt updates to be guided by increasingly unreliable supervision.

Motivated by this intuition, we propose an instance-grounded semantic anchoring principle for prompt learning in the presence of label noise. \textit{Prompt adaptation should be anchored to the semantic content of the input image, rather than relying solely on the observed noisy label}. This principle is based on a simple intuition that, although annotations may be corrupted, the image itself still preserves instance-specific semantic evidence that can provide a more stable guidance signal for prompt optimization.

To instantiate this principle, we propose \textbf{VisPrompt} (Vision-Guided Cross-Modal Prompt Learning under Label Noise), a robust prompt learning framework that explicitly injects image-grounded semantics into prompt adaptation in noisy-label settings. VisPrompt first introduces {\textit{cross-modal attention}} to inject visual evidence into the learnable prompts. Instead of updating prompts solely according to label supervision, our method allows prompt tokens to selectively aggregate semantically relevant information from the images. In this way, the prompt representation is no longer optimized only toward the observed noisy objective, but is also constrained by the semantics of the current instance. As illustrated in Fig.~\ref{fig:teaser}, such image-grounded conditioning helps steer the prediction toward the correct class even when the assigned label is corrupted. However, the usefulness of visual evidence is not uniform across samples. Applying the same fusion strength to all instances may underuse highly informative visual cues, while also introducing unstable or irrelevant perturbations from low-quality ones. To address this issue, we further introduces {\textit{FiLM (Feature-wise Linear Modulation)}} as an instance-adaptive control mechanism. Conditioned on the current visual representation, FiLM adaptively modulates how much visual evidence should be injected into the prompt and which feature dimensions should be emphasized or suppressed. This selective modulation enables a more controlled integration of image-side evidence, leading to more stable prompt updates under noisy supervision.
From a robustness perspective, cross-modal attention and FiLM modulation selectively introduce and regulate visual evidence, thereby improving the stability of prompt optimization under noisy supervision. Importantly, VisPrompt improves robustness with minimal trainable overhead, introducing less than 1\% additional parameters while keeping the pretrained VLM backbone frozen. Our main contributions are summarized as follows:

\begin{itemize}[leftmargin=0.3cm]

\item We propose \emph{VisPrompt, a lightweight and robust prompt learning framework for noisy-label settings.} It reformulates prompt adaptation from a cross-modal perspective and exploits image-grounded semantics to guide prompt learning. This reduces the interference of corrupted labels during optimization.

\item We introduce a vision-guided prompt modulation mechanism that combines cross-modal attention with FiLM gating. Cross-modal attention \emph{extracts visual evidence that is more informative and reliable.} FiLM then adaptively \emph{controls how this evidence is injected into the prompt representation.} This enables selective enhancement of reliable cues and suppression of unstable perturbations.

\item We validate the robustness of VisPrompt through \emph{theoretical analysis.} The analysis shows that cross-modal attention yields a denoised approximation of clean semantics. It also shows that FiLM injects such evidence into prompts in a bounded and stable manner. Extensive experiments further demonstrate that VisPrompt achieves competitive or superior performance compared with representative robust prompt learning based methods on seven datasets with synthetic and real-world label noise.

\end{itemize}

\section{Related Work}
\label{sec:related_work}
\subsection{Prompt Learning}
With the rapid advance of vision-language models, prompt learning emerges as a pivotal research direction as a parameter-efficient learning approach. Numerous studies \cite{cai2025attributebasedvisualreprogrammingvisionlanguage, cai2024samplespecificmasksvisualreprogrammingbased,CLIP} demonstrate its efficiency and scalability, with CLIP \cite{CLIP} being the most representative work. Early prompt learning relies on manually designed templates, such that minor changes could result in significant performance variation. CoOP \cite{CoOp} introduces learnable continuous context on the text side and optimizes end to end with a frozen backbone, significantly improving adaptation efficiency. CoCoOp \cite{zhouConditionalPromptLearning2022} leverages image conditioned context to enhance generalization to unseen classes, and MaPLe \cite{khattakMapleMultimodalPrompt2023} extends prompt integration to deeper layers, coupling it along the entire vision-language pathway. KAPT \cite{Kan2023KnowledgeAwarePT} and ATPrompt \cite{li2025advancing} introduce external attribute knowledge and generic attribute tokens, respectively, to strengthen cross-domain generalization and class alignment. While these methods improve accuracy and transferability given clean supervision, their capability to handle noisy labels remains rarely explored.

\subsection{Noisy Label Learning}

In real-world datasets, label noise is inevitable, and corrupted annotations can cause severe performance degradation. To address this issue, numerous previous works have been proposed from diverse aspects, including robust loss functions \cite{fengCanCrossEntropy2021,lyuCurriculumLossRobust2019,ghoshRobustLossFunctions2017,fedrom}, robust regularization \cite{hendrycksUsingPretrainingCan2019,menonCanGradientClipping2020,xiaRobustEarlylearningHindering2020, jiang2022towards}, sample selection \cite{wei2024deft,fedrom} and meta-learning \cite{liDividemixLearningNoisy2020,patelAdaptiveSampleSelection2023,songSelfieRefurbishingUnclean2019},  loss correction \cite{changActiveBiasTraining2017,xiaAreAnchorPoints2019,yaoDualReducingEstimation2020}, and robust training framework design \cite{leeRobustInferenceGenerative2019,yaoDeepLearningNoisy2018,jiang2024tackling}. However, these works mainly focus on the unimodal tasks like image or text classification. 
In prompt learning, however, research on handling noisy labels remains limited. PTNL \cite{wuWhyPromptTuning2023} opens this direction by revealing the robustness potential of prompt learning. GCE \cite{GCE} enhances the robustness of prompt learning from the perspective of loss function design. JoAPR \cite{Guo2024JoAPR} fits a two-component Gaussian mixture to the loss distribution and uses adaptive thresholds for clean-sample selection and relabeling. NLPrompt \cite{Pan2025NLPrompt} employs optimal transport to partition clean and noisy subsets, while TrustCLIP \cite{TrustCLIP} estimates trustworthiness from the structural semantics learned during training. Although these methods achieve meaningful gains, they mostly rely on auxiliary confidence modeling or sample partition strategies, and largely optimize robustness from the text or label side. As a result, they do not fully exploit image content itself as a relatively stable and instance-specific source of supervision under label corruption.

\subsection{Vision-guided prompt generation}
Beyond text-only prompt learning methods such as CoOp \cite{CoOp}, recent studies have explored incorporating visual cues into prompt learning for adapting vision-language models. CoCoOp \cite{zhouConditionalPromptLearning2022} conditions prompts on individual input images, while subsequent methods such as DPT \cite{Ranftl2021}, further strengthen cross-modal interaction by jointly leveraging visual and textual information. Other works, such as GalLoP \cite{lafon2024gallop} , exploit local visual evidence or attribute-level semantics to improve prompt quality and generalization. These studies collectively suggest that visual guidance can enrich prompt learning with instance-specific semantics and improve adaptation performance. 
While existing approaches have mitigated label noise to some extent, they predominantly focus on loss function design, noisy-sample detection, and data purification, often overlooking the potential of visually guided prompt generation. To bridge this gap, we propose VisPrompt, a vision-guided framework for few-shot learning under noisy label conditions.

\section{Preliminary}

We first introduce the standard prompt learning paradigm for vision-language pre-trained models. Given an input image, the image encoder and text encoder project visual and textual inputs into a shared embedding space, where classification is performed according to their similarity.

For Vision Encoder, given an input image $I$, the image encoder first divides it into $M$ patches and maps them into patch embeddings, denoted by
$V^{0}=\{v_1^{0}, \dots, v_M^{0}\}$.
A learnable class token $v_{\mathrm{cls}}^{0}$ is then appended to the patch sequence.
The resulting tokens are fed into a stack of vision transformer layers:

$[V^{l}, v_{\mathrm{cls}}^{l}] = \mathrm{Image}^{l}\bigl([V^{l-1}, v_{\mathrm{cls}}^{l-1}]\bigr),\quad l=1,\dots,L_I.$
where $L_I$ denotes the number of layers in the visual encoder.
The final image representation in the shared vision-language space is obtained by projecting the class token from the last layer:
$h = f_{\mathrm{Proj}} \bigl(v_{\mathrm{cls}}^{L_I}\bigr).$

For Text Encoder, given a text prompt, the text encoder tokenizes it into $N$ tokens, denoted by
$C^{0}=\{c_1^{0}, \dots, c_N^{0}\}$.
These token embeddings are processed by a stack of text transformer layers:
$T^{l} =
\mathrm{Text}^{l}(C^{l-1}),
\quad l=1,\dots,L_T,$
where $L_T$ is the number of layers in the text encoder.
The final text feature is produced by projecting the last token at the top layer into the shared vision-language space:
$g = f_{\mathrm{Proj}} \bigl(c_N^{L_T}\bigr).$
In standard prompt learning, the parameters of both the image encoder and the text encoder are kept frozen, and only a small set of learnable context tokens is optimized.
For a class name represented by $[{\rm CLASS}]$, the prompt is constructed by combining it with $n$ learnable context tokens, denoted by
$P=\{p_1,\dots,p_n\}$.
The resulting prompt can be formulated as
\begin{equation}
\mathcal{C}_k = [p_1, p_2, \dots, p_n, \mathrm{class}_k],
\end{equation}
where $\mathrm{class}_k$ denotes the tokenized name of the $k$-th class.
The prompt is then fed into the text encoder to obtain the corresponding class-level text feature:
\begin{equation}
g_k = \mathrm{TextEncoder}(\mathcal{C}_k), \quad k=1,\dots,K,
\end{equation}
where $K$ is the number of classes.

Given an input image, the image encoder produces image feature $h$, while the text encoder generates a text prototype $g_k$ for each class.
Result is performed in the shared embedding space according to the similarity between $h$ and $g_k$:
\begin{equation}
p(y=k \mid I)=
\frac{\exp \bigl(\mathrm{sim}(h,g_k)/\tau \bigr)}
{\sum_{j=1}^{K}\exp \bigl(\mathrm{sim}(h,g_j)/\tau \bigr)},
\end{equation}
where $\mathrm{sim}(\cdot,\cdot)$ denotes cosine similarity and $\tau$ is a temperature parameter.
The model is trained by updating the learnable prompt tokens such that the text representation of the correct class becomes better aligned with the visual representation of the input image.

\section{Methodology}

\label{sec:method_pipeline}
\begin{figure*}[htbp]
\centering
\includegraphics[width=\linewidth]{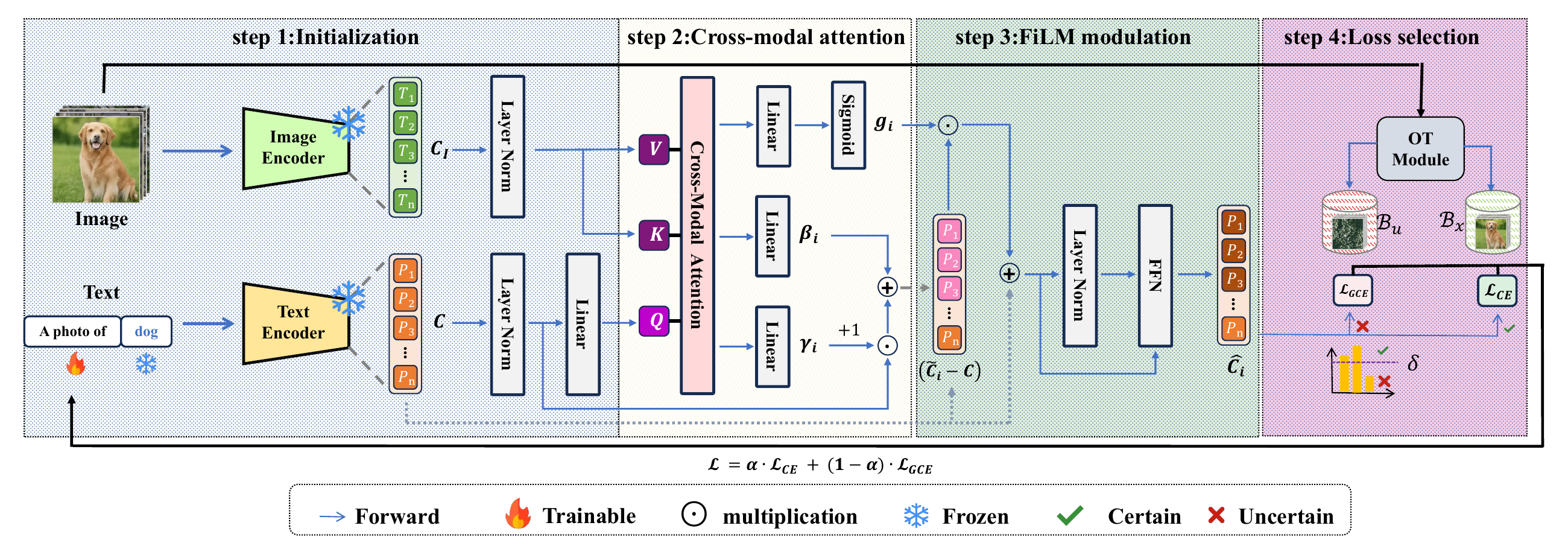}
 \caption{The overall architecture of VisPrompt which consists of four steps. Step 2 and Step 3 denote the robust update process.
    }
 \label{fig:frame}
\end{figure*}
To improve the robustness of prompt learning under noisy supervision, we propose \textbf{VisPrompt}, a visual-guided prompt learning framework for noise-robust adaptation. As illustrated in Fig.~\ref{fig:frame}, the framework consists of two key components: \textbf{(i)} Cross-modal Visual Prompt Conditioning, which injects multiple local visual tokens into the prompt context through cross-modal attention; \textbf{(ii)} FiLM-based Robust Modulation, which adaptively controls the injected visual guidance through conditional modulation and residual gating. By combining reliable image-grounded semantics with the generic prior of the text branch, VisPrompt mitigates noise-induced prompt drift and yields more stable and robust prompt optimization.

\subsection{VisPrompt Framework}
\label{sec:visprompt_framework}

VisPrompt is instantiated through four sequential stages: 
Initialization, cross-modal visual prompt conditioning, FiLM-based robust modulation, and robust loss selection.

\subsubsection{Learnable Context Initialization}

Within the prompt learning framework, we introduce $n_{\text{ctx}}$ learnable context tokens as the optimizable part of the prompt. 
Depending on the parameterization strategy, the context can be either class-shared or class-specific. 
For the class-shared setting, the learnable context is denoted as
\begin{equation}
\mathbf{C} \in \mathbb{R}^{n_{\text{ctx}} \times d},
\label{eq:ctx_shared}
\end{equation}
where $d$ is the dimensionality of the text embedding space. 
For the class-specific setting, the context can be extended as
\begin{equation}
\mathbf{C} \in \mathbb{R}^{N \times n_{\text{ctx}} \times d},
\label{eq:ctx_specific}
\end{equation}
where $N$ denotes the number of classes.

\subsubsection{Visual Feature Projection}
Unlike class labels, which may be corrupted at the sample level, local image content still preserves fine-grained semantic evidence from the underlying visual instance. To explicitly exploit such relatively reliable instance-level cues, we extract a set of local visual representations from the input image. Given an input image $x_i$, the image encoder produces
\begin{equation}
\mathbf{V}_i \in \mathbb{R}^{M \times d_v},
\label{eq:visual_tokens}
\end{equation}
where $M$ is the number of visual tokens and $d_v$ is the dimensionality of the visual feature space. 
Since the visual and textual branches generally reside in different representation spaces, we employ a learnable linear projection matrix
\begin{equation}
\mathbf{W}_p \in \mathbb{R}^{d_v \times d}
\label{eq:projection_matrix}
\end{equation}
to align visual features with the prompt embedding space, yielding
\begin{equation}
\mathbf{Z}_i = \mathbf{V}_i \mathbf{W}_p \in \mathbb{R}^{M \times d}.
\label{eq:projected_visual}
\end{equation}

Projecting these tokens into the text-aligned space allows the model to access image-grounded semantics that are less affected by annotation errors. 
As a result, the prompt is not forced to rely solely on potentially misleading supervision, but can instead condition on more reliable instance-level cues.

\subsubsection{Cross-modal Visual Prompt Conditioning}

To reduce the direct influence of noisy labels on prompt updates, this module is designed to introduce an \emph{instance-specific image-conditioned signal} into the learnable context. 
Instead of letting the prompt be updated solely by the supervisory signal of the assigned label, we explicitly establish cross-modal interactions between the context tokens and the projected visual tokens, so that the prompt can be conditioned on the semantic content of the current image. 
Let $\mathbf{C}$ denote the context tokens and $\mathbf{Z}_i$ denote the projected visual token sequence associated with image $x_i$. 
The cross-modal conditioning feature is defined as
\begin{equation}
\mathbf{A}_i = \mathrm{MHA}\bigl(\mathrm{LN}(\mathbf{C}), \mathrm{LN}(\mathbf{Z}_i), \mathrm{LN}(\mathbf{Z}_i)\bigr),
\label{eq:cross_attention}
\end{equation}
where $\mathrm{MHA}(\cdot)$ and $\mathrm{LN}(\cdot)$ denote multi-head attention and layer normalization, respectively.

Here, $\mathbf{A}_i$ is constructed as the cross-modal conditioning representation for the current sample. 
Its role is to use the context tokens as queries to selectively retrieve semantically relevant local evidence from the visual tokens, and to organize such image-grounded information into a conditioning signal for subsequent prompt modulation.

\subsubsection{FiLM-based Robust Modulation}

Although visual evidence is generally more reliable than noisy labels, its quality may still vary across instances. 
To incorporate such instance-dependent visual cues in a controlled manner, we employ a FiLM-based modulation mechanism together with token-wise gating. 
Based on the cross-modal feature $\mathbf{A}_i$, we first generate FiLM modulation parameters to perform feature-wise adjustment on the normalized context:
\begin{equation}
\tilde{\mathbf{C}}_i
=
\mathrm{LN}(\mathbf{C}) \odot \bigl(1 + \phi_{\gamma}(\mathbf{A}_i)\bigr)
+
\phi_{\beta}(\mathbf{A}_i),
\label{eq:film_modulation}
\end{equation}
and further compute a token-wise gate as
\begin{equation}
\mathbf{G}_i
=
\sigma\bigl(\phi_{g}(\mathbf{A}_i)\bigr),
\label{eq:gate}
\end{equation}
where $\phi_{\gamma}(\cdot)$, $\phi_{\beta}(\cdot)$, and $\phi_{g}(\cdot)$ are learnable mappings for the scaling term, bias term, and gate generation, respectively, $\sigma(\cdot)$ denotes the sigmoid function, and $\odot$ denotes element-wise multiplication.

Given the modulated context $\tilde{\mathbf{C}}_i$ and the token-wise gate $\mathbf{G}_i$, we then perform a residual gated update:
\begin{equation}
\mathbf{C}'_i
=
\mathbf{C}
+
\mathbf{G}_i \odot \bigl(\tilde{\mathbf{C}}_i - \mathbf{C}\bigr),
\label{eq:residual_update}
\end{equation}
followed by a feed-forward refinement:
\begin{equation}
\hat{\mathbf{C}}_i
=
\mathbf{C}'_i + \mathrm{FFN}\bigl(\mathrm{LN}(\mathbf{C}'_i)\bigr),
\label{eq:ffn_refine}
\end{equation}
The resulting $\hat{\mathbf{C}}_i$ is used as the image-conditioned prompt context for the current sample.

\subsubsection{Robust Loss Design}

The impact of label noise varies across samples. 
Clean samples usually provide relatively reliable supervision, whereas noisy samples are more likely to be dominated by corrupted annotations. 
When all samples are optimized with the same loss function, the model may gradually overfit noisy labels and thus lose robustness. 

To alleviate this issue, this work introduces an Optimal Transport (OT) \cite{OT} mechanism to estimate sample reliability under noisy supervision. 
According to the estimated reliability, the training set is partitioned into a reliable subset and an unreliable subset. 
Different loss functions are then assigned to these two subsets. 
This strategy allows the optimization process to better adapt to the heterogeneous quality of supervisory signals. 

The central function of OT is to determine a global transport plan between two predefined distributions under a given transport cost and marginal constraints. 
By jointly modeling all sample class matching relationships, OT can suppress locally abnormal assignments and produce a more globally consistent correspondence structure. 

Assume that the training set consists of $N$ samples and $L$ classes. The frozen image encoder extracts the image feature matrix $V \in \mathbb{R}^{N \times d}$, where $d$ denotes the feature dimension, while the text encoder produces the class text feature matrix $T \in \mathbb{R}^{L \times d}$.

The similarity score between each sample and each class is then computed. 
The normalized similarity is further converted into the transport cost used in OT. 
Specifically, the normalized similarity between sample $i$ and class $j$ is defined as
\begin{equation}
S_{ji}
=
\frac{\exp\!\left(\mathrm{sim}(t_j, v_i)/\tau\right)}
{\sum_{k=1}^{L}\exp\!\left(\mathrm{sim}(t_k, v_i)/\tau\right)},
\label{eq:ot_similarity}
\end{equation}
where $\tau$ is the temperature coefficient and $\mathrm{sim}(\cdot,\cdot)$ denotes the similarity function. 
The corresponding matching cost is defined by
\begin{equation}
D_{ji} = -\log S_{ji}.
\label{eq:ot_cost}
\end{equation}
The transport plan is obtained by solving the entropic OT problem
\begin{equation}
\Pi^{*}
=
\arg\min_{\Pi \in \mathcal{U}(a,b)}
\langle \Pi, D \rangle
-
\varepsilon H(\Pi),
\label{eq:ot_sinkhorn}
\end{equation}
where $\mathcal{U}(a,b)$ denotes the set of admissible transport plans with prescribed marginals $a$ and $b$, $\varepsilon$ is the entropic regularization coefficient, and $H(\Pi)$ denotes the entropy term. The optimal transport plan $\Pi^{*}$ is efficiently computed by the Sinkhorn algorithm.

Based on $\Pi^{*}$, a reliability score is assigned to each sample. For sample $i$, the OT-induced pseudo-label is first defined as
\begin{equation}
\hat{y}_i = \arg\max_{j} \Pi^{*}_{ji}.
\label{eq:ot_pseudo_label}
\end{equation}
The corresponding confidence score is then defined as
\begin{equation}
r_i = \Pi^{*}_{\hat{y}_i i}.
\label{eq:ot_confidence}
\end{equation}
A sample is regarded as reliable if its OT assignment is consistent with the observed label and its confidence exceeds a predefined threshold $\delta$. Accordingly, the reliable subset and unreliable subset are defined as
\begin{equation}
\mathcal{B}_{x}
=
\left\{
(x_i,y_i)\; \middle|\; \hat{y}_i = y_i,\; r_i \ge \delta
\right\}.
\label{eq:reliable_set}
\end{equation}

\begin{equation}
\mathcal{B}_{u}
=
\left\{
(x_i,y_i)\; \middle|\; \hat{y}_i \neq y_i \;\text{or}\; r_i < \delta
\right\}.
\label{eq:unreliable_set}
\end{equation}
The reliable subset is optimized with standard cross-entropy loss,
\begin{equation}
\mathcal{L}_{\mathrm{CE}}
=
-
\frac{1}{|\mathcal{B}_{x}|}
\sum_{(x_i,y_i)\in \mathcal{B}_{x}}
\log p_{i,y_i},
\label{eq:ce_reliable}
\end{equation}
where $p_{i,y_i}$ denotes the predicted probability of sample $i$ on its annotated class $y_i$. For the unreliable subset, the generalized cross-entropy loss is adopted to reduce the adverse effect of corrupted labels:
\begin{equation}
\mathcal{L}_{\mathrm{GCE}}
=
\frac{1}{|\mathcal{B}_{u}|}
\sum_{(x_i,y_i)\in \mathcal{B}_{u}}
\frac{1 - p_{i,y_i}^{\,q}}{q},
\qquad q \in (0,1].
\label{eq:gce_unreliable}
\end{equation}

The overall robust training objective is therefore formulated as
\begin{equation}
\mathcal{L}_{\mathrm{robust}}
=
\alpha \mathcal{L}_{\mathrm{CE}}
+
(1 - \alpha) \mathcal{L}_{\mathrm{GCE}},
\label{eq:robust_loss}
\end{equation}
where $\alpha$ controls the extent of the unreliable subset for optimization.

\textbf{Remark.}
The above design preserves parameter-efficient for prompt learning: The pre-trained image encoder and text encoder remain frozen, while only the learnable prompt context and the lightweight visual-guided modulation modules are optimized.

\subsection{Theoretical Analysis}
\label{sec:theory}

Herein we explain why VisPrompt is robust under corrupted supervision from two aspects: Cross-modal attention can extract reliable semantic signals grounded in image content, while FiLM-based gated modulation injects this signal into the prompt in a controlled manner rather than  uncontrolled prompt drift.

\subsubsection{Basic assumptions}
For each image $x_i$, let $s_i$ denote the latent clean semantic representation of the instance in the shared vision-language space. 
Assume that: 
(i) among the $M$ visual tokens, there exists an informative subset whose elements are close to $s_i$; 
(ii) cross-modal attention assigns larger scores to informative tokens than to irrelevant ones, with a positive margin $\Delta_i$; 
(iii) the FiLM and gating mappings are Lipschitz continuous \cite{lips}, and the residual gate is bounded in $[0,1]$. 
These assumptions are consistent with the roles of shared image-text representation, instance-conditioned prompts, and feature-wise modulation in prior studies~\cite{perez2018film,zhouConditionalPromptLearning2022}.

\subsubsection{Cross-modal semantic aggregation}
The cross-modal attention output is formulated as
\begin{equation}
A_i=\sum_{m=1}^{M}\alpha_{i,m} z_{i,m},
\qquad
\alpha_{i,m}=\frac{\exp(a_{i,m})}{\sum_{r=1}^{M}\exp(a_{i,r})},
\label{eq:theory_attn}
\end{equation}
where $a_{i,m}$ is the attention score between the prompt query and the $m$-th visual token. 
Since informative tokens are assumed to have a score margin over irrelevant ones, the softmax weights concentrate on semantically relevant local evidence. 
As a result, the aggregated feature $A_i$ serves as a denoised approximation of the clean instance semantics $s_i$, and the approximation error decreases as the attention margin increases.

\subsubsection{Controlled prompt modulation}
Let the FiLM-based robust modulation be denoted compactly by
\begin{equation}
\hat{\mathbf{C}}_i=\mathcal{F}(\mathbf{C},A_i),
\label{eq:theory_modulation}
\end{equation}
where $\mathcal{F}$ summarizes the FiLM transformation, token-wise gating, residual update, and FFN refinement defined in Eqs.~(12)--(15). 
Denote by $\hat{\mathbf{C}}_i^\star=\mathcal{F}(\mathbf{C},s_i)$ the ideal prompt that would be obtained if the clean semantic signal $s_i$ were directly available. 
Because FiLM is feature-wise and the residual gate is bounded, $\mathcal{F}$ is a stable Lipschitz mapping with respect to its conditioning signal, which means that perturbations in $A_i$ cannot be arbitrarily amplified during prompt modulation.

\begin{theorem}
Under the above assumptions, there exists a constant $L_{\mathrm{mod}}>0$ such that
\begin{equation}
\|\hat{\mathbf{C}}_i-\hat{\mathbf{C}}_i^\star\|
\le
L_{\mathrm{mod}}
\Bigl(
\varepsilon_v + c_i e^{-\Delta_i}
\Bigr),
\label{eq:main_theorem}
\end{equation}
where $\varepsilon_v$ measures the approximation error between informative visual tokens and the clean instance semantics, and $c_i e^{-\Delta_i}$ is the residual distraction term induced by irrelevant tokens.
\end{theorem}

\textbf{Proof}
The proof follows two steps. 
First, by the positive attention margin, the softmax mass assigned to irrelevant visual tokens decays exponentially with $\Delta_i$, so the aggregated feature $A_i$ remains close to the clean semantic signal $s_i$. 
Second, since FiLM and gate generation are Lipschitz \cite{lips} and the residual gate is bounded, the modulation map $\mathcal{M}$ transmits this perturbation in a controlled way rather than amplifying it, which directly yields Eq.~\eqref{eq:main_theorem}.

\subsubsection{Implication for robustness.}
Eq.~\eqref{eq:main_theorem} shows that the deviation of the image-conditioned prompt from its ideal clean-semantic counterpart is jointly controlled by two factors: the quality of informative local visual evidence and the attention margin over irrelevant tokens. 
Therefore, compared with purely label-driven prompt updates, the proposed design is less sensitive to corrupted supervision: Cross-modal attention suppresses noisy or irrelevant evidence at the aggregation stage, while FiLM-based gated modulation prevents such perturbations from causing excessive prompt drift.

Furthermore, if the clean-prompt logit margin is sufficiently large, the final prediction remains unchanged:
\begin{equation}
m_i^\star
>
2L_h L_{\mathrm{mod}}
\Bigl(
\varepsilon_v + c_i e^{-\Delta_i}
\Bigr),
\label{eq:margin_condition}
\end{equation}
where $m_i^\star$ is the classification margin under the ideal prompt and $L_h$ is the Lipschitz constant of the frozen classifier with respect to the prompt representation. 
This condition indicates that the class decision is preserved as long as the perturbation induced by noisy supervision stays below the clean margin.

\section{Experiments }
\label{sec:experienment}

\subsection{Datasets with Noisy Labels} 
We evaluate our method on six benchmark datasets with synthetic symmetric label noise: EuroSAT \cite{eurosat}, Flowers102 \cite{flowers102}, OxfordPets \cite{oxfordpets}, DTD \cite{dtd}, UCF101 \cite{ucf101}, and Caltech101 \cite{caltech101}. These datasets cover diverse recognition objectives, and their statistics are summarized in Table~\ref{tab:datasets}. Owing to their reliable annotations, they provide a suitable testbed for controlled synthetic noise injection under the few-shot setting.
Meanwhile, we further conduct evaluation on Food101N \cite{food101n}, a real-world noisy dataset collected from web sources. Unlike synthetic corruption, its label noise arises naturally from weak supervision and imperfect data collection.

Regarding the synthetic label noise patterns, we introduce both symmetric and asymmetric label noise. For symmetric noise, each class label is flipped to any other class with equal probability. In contrast, asymmetric noise assigns different flip probabilities to each class, better mirroring real-world conditions. We consider six label noise rates at 12.5\%, 25\%, 37.5\%, 50\%, 62.5\%, and 75\% to provide broad coverage of plausible noise rates.

\begin{table}[htbp]
\centering
\caption{Dataset Statistics.}
\label{tab:datasets}
\small
\setlength{\tabcolsep}{2.5pt}
\renewcommand{\arraystretch}{1.05}

\begin{tabular}{@{}lccccp{1.8cm}@{}}
\toprule
Dataset & Classes & Train & Test & Noise Type & Objective \\
\midrule
Caltech101 \cite{caltech101} & 100 & 4,128   & 2,465  & Synthetic  & Objects \\
Flowers102 \cite{flowers102} & 102 & 4,093   & 2,463  & Synthetic  & Flowers \\
OxfordPets \cite{oxfordpets} & 37  & 2,944   & 3,669  & Synthetic  & Pets \\
UCF101 \cite{ucf101}         & 101 & 7,639   & 3,783  & Synthetic  & Human actions \\
DTD \cite{dtd}               & 47  & 2,820   & 1,692  & Synthetic  & Textures \\
EuroSAT \cite{eurosat}       & 10  & 13,500  & 8,100  & Synthetic  & Satellite scenes \\ \midrule
Food101N \cite{food101n}     & 101 & 310,009 & 30,300 & Real-world & Food categories \\
\bottomrule
\end{tabular}
\end{table}

\subsection{Baselines and Implementation Details}

In this study, our baselines for comparison include CoOP \cite{CoOp}, GCE \cite{GCE}, JoAPR \cite{Guo2024JoAPR}, and NLPrompt \cite{Pan2025NLPrompt}, which are representative methods for robust prompt learning under noisy labels. 
We adopt the ResNet-50 (RN50) \cite{He2015DeepRL} and ViT-B/16 \cite{50650} as the image encoders. Unless otherwise clarified, we report the evaluation results with RN50 by default. For the fair comparison with baselines, we adopt the same hyperparameter settings as CoOP \cite{CoOp}, JoAPR \cite{Guo2024JoAPR}, and NLPrompt \cite{Pan2025NLPrompt} with the SGD optimizer with a learning rate of 0.002 under a cosine annealing scheduler. All main experiments are conducted on an Nvidia RTX 4090 GPU and PyTorch framework \cite{paszkePytorchImperativeStyle2019}. We use accuracy as the main evaluation metric. Results are averaged over 3 different random seeds to ensure the fairness. Mixed precision training \cite{mixedP} is used to accelerate the training process. By default, the prompt length is 16, the number of attention heads is 8, no class-specific initialization is applied, and the number of shots is 16. The training batch size is 16 for all methods. We perform 200 training epochs to ensure the convergence.

\begin{table*}[htbp]
\centering
\caption{Performance (\%) over different datasets with varying noise rates. The \textbf{bold} denotes the best performance.}
\label{tab:main}
\resizebox{\textwidth}{!}{
\begin{tabular}{c|c|cccccc|c|cccccc|c}
\toprule 
\multirow{2}{*}{\textbf{Dataset}} & \multirow{2}{*}{\textbf{Method}} & \multicolumn{6}{c|}{\textbf{Symmetric Label Noise}} & \multirow{2}{*}{\textbf{avg.}} & \multicolumn{6}{c|}{\textbf{Asymmetric Label Noise}} & \multirow{2}{*}{\textbf{avg.}} \\
\cmidrule{3-8} \cmidrule{10-15}
 & & 12.5\% & 25.0\% & 37.5\% & 50.0\% & 62.5\% & 75.0\% & & 12.5\% & 25.0\% & 37.5\% & 50.0\% & 62.5\% & 75.0\% & \\
\midrule

\multirow{4}{*}{Flowers102 \cite{flowers102}} & CoOp \cite{CoOp} & 88.93 & 83.50 & 77.93 & 70.10 & 55.60 & 37.17 & 68.87  &86.97 & 74.70 & 60.43 & 42.60 & 26.53 & 12.60 &50.64 \\
& GCE \cite{GCE} & 88.80 & 88.33 & 86.73 & 84.07 & 78.37 & 70.37 & 82.78& 88.40 & 86.37 & 80.33 & 69.93 & 61.50 & 39.23&70.96 \\
& JoAPR \cite{Guo2024JoAPR} & 85.57 & 81.23 & 74.60 & 70.23 & 67.90 & 66.93 & 74.41& 85.17 & 79.63 & 73.97 & 73.83 & 53.37 & 13.27&63.21 \\
& NLPrompt \cite{Pan2025NLPrompt}  & 93.87 & 92.57 & 92.73 & 89.90 & 84.77 & 76.80 & 88.44& 93.80 & 93.40 & \textbf{91.77} & 81.10 & 73.63 & 55.33&81.51 \\

& \cellcolor{blue!10}\textbf{VisPrompt} & \cellcolor{blue!10}\textbf{95.57} & \cellcolor{blue!10}\textbf{95.29} & \cellcolor{blue!10}\textbf{93.99} & \cellcolor{blue!10}\textbf{92.37} & \cellcolor{blue!10}\textbf{89.77} & \cellcolor{blue!10}\textbf{77.73} & \cellcolor{blue!10}\textbf{90.79} & \cellcolor{blue!10}\textbf{95.33} & \cellcolor{blue!10}\textbf{93.87} & \cellcolor{blue!10}89.61 & \cellcolor{blue!10}\textbf{83.84} & \cellcolor{blue!10}\textbf{77.69} & \cellcolor{blue!10}\textbf{55.63} & \cellcolor{blue!10}\textbf{82.66} \\ \midrule

\multirow{4}{*}{DTD \cite{dtd}} & CoOp \cite{CoOp} & 56.00 & 49.57 & 43.30 & 34.37 & 27.83 & 17.27 &38.06& 55.60 & 47.75 & 38.07 & 29.63 & 20.53 & 11.70&33.88 \\
& GCE \cite{GCE}& 61.00 & 59.83 & 56.80 & 50.73 & 43.60 & 33.67 &50.94& 60.70 & 57.57 & 52.70 & 43.97 & 33.40 & 18.23&44.43\\
& JoAPR \cite{Guo2024JoAPR}& 58.07 & 57.70 & 56.33 & 53.03 & 48.05 & 29.90 &50.51& 52.40 & 56.63 & 53.10 & 48.93 & 40.20 & 28.26&46.59  \\
& NLPrompt \cite{Pan2025NLPrompt}  & 62.97 & 61.23 & 59.17 & 55.17 & 49.03 & 39.80 &54.56& 62.30 & \textbf{60.60} & 56.47 & 50.80 & 40.27 & 28.37&49.80  \\

& \cellcolor{blue!10}\textbf{VisPrompt} & \cellcolor{blue!10}\textbf{66.55} & \cellcolor{blue!10}\textbf{62.71} & \cellcolor{blue!10}\textbf{62.47} & \cellcolor{blue!10}\textbf{56.32} & \cellcolor{blue!10}\textbf{53.55} & \cellcolor{blue!10}36.89 & \cellcolor{blue!10}\textbf{56.42} & \cellcolor{blue!10}\textbf{66.49} & \cellcolor{blue!10}59.83 & \cellcolor{blue!10}\textbf{57.21} & \cellcolor{blue!10}\textbf{52.07} & \cellcolor{blue!10}\textbf{43.20} & \cellcolor{blue!10}\textbf{33.22} & \cellcolor{blue!10}\textbf{52.00} \\ \midrule

\multirow{4}{*}{EuroSAT \cite{eurosat}} & CoOp \cite{CoOp} & 76.50 & 69.23 & 61.67 & 52.33 & 37.63 & 26.70 &54.01 &76.00 & 66.27 & 53.83 & 41.17 & 28.00 & 17.43&47.12 \\
& GCE \cite{GCE} & 82.13 & 78.60 & 74.67 & 63.13 & 49.67 & 31.40 &63.27& 78.23 & 72.70 & 63.63 & 45.30 & 22.90 & 12.10&49.14\\
& JoAPR \cite{Guo2024JoAPR} & 75.13 & 61.10 & 60.90 & 63.63 & 38.97 & 27.33 &54.51& 69.37 & 67.30 & 59.40 & 47.60 & 33.93 & 17.50 &49.18\\
& NLPrompt \cite{Pan2025NLPrompt}  & 81.4 & 76.51 & 75.58 & 65.02 & 58.40 & 31.00 &64.65& 80.01 & 77.4 & \textbf{70.50} & 54.30 & 23.93 & 13.01&53.19 \\

& VisPrompt\cellcolor{blue!10}& \cellcolor{blue!10}\textbf{82.83} & \cellcolor{blue!10}\textbf{80.58} & \cellcolor{blue!10}\textbf{76.90} & \cellcolor{blue!10}\textbf{67.70} &\cellcolor{blue!10} \textbf{63.49} &\cellcolor{blue!10} \textbf{36.35} &\cellcolor{blue!10}\textbf{67.98}& \cellcolor{blue!10}\textbf{81.95} & \cellcolor{blue!10}\textbf{80.43} &\cellcolor{blue!10} 69.16 & \cellcolor{blue!10}\textbf{54.52} & \cellcolor{blue!10}\textbf{34.78} & \cellcolor{blue!10}\textbf{30.72}&\cellcolor{blue!10}\textbf{58.59} \\ \midrule

\multirow{4}{*}{OxfordPets \cite{oxfordpets}} & CoOp \cite{CoOp} & 76.50 & 66.73 & 60.33 & 47.03 & 35.77 & 24.60 &51.83& 76.10 & 66.20 & 52.53 & 38.73 & 26.63 & 14.90 &45.85 \\
& GCE \cite{GCE}& 85.63 & 84.60 & 83.67 & 79.23 & 71.40 & 53.17 &76.28& 85.50 & 83.03 & 76.73 & 68.07 & 50.70 & 31.97&66.00 \\
& JoAPR \cite{Guo2024JoAPR} & 84.00 & 83.26 & 83.20 & 83.10 & 82.40 & 74.40 &81.73& 82.90 & 83.40 & 79.07 & 75.84 & 52.74 & 43.57&69.59 \\
& NLPrompt \cite{Pan2025NLPrompt}  & 86.13 & 83.53 & 83.51 & 81.51 & 78.74 & 59.41 &78.81& 85.93 & 84.22 & 79.31 & 77.62 & 61.94 & 36.71&70.96 \\

& VisPrompt \cellcolor{blue!10} & \cellcolor{blue!10}\textbf{89.94} & \cellcolor{blue!10}\textbf{89.56} & \cellcolor{blue!10}\textbf{89.78} & \cellcolor{blue!10}\textbf{89.75} & \cellcolor{blue!10}\textbf{88.80} & \cellcolor{blue!10}\textbf{87.14} &\cellcolor{blue!10}\textbf{89.16}& \cellcolor{blue!10}\textbf{90.00} & \cellcolor{blue!10}\textbf{89.67} & \cellcolor{blue!10}\textbf{89.59} & \cellcolor{blue!10}\textbf{89.48} &\cellcolor{blue!10} \textbf{88.74} & \cellcolor{blue!10}\textbf{86.75}&\cellcolor{blue!10}\textbf{89.04} \\ \midrule

\multirow{4}{*}{UCF101 \cite{ucf101}} & CoOp \cite{CoOp} & 69.03 & 63.40 & 58.23 & 49.73 & 40.83 & 26.30 &51.25& 67.23 & 58.07 & 46.47 & 34.43 & 23.67 & 13.17&40.51 \\
& GCE \cite{GCE}& 74.00 & \textbf{73.63} & 72.57 & 69.37 & 66.00 & 57.07 &68.77& 73.90 & 71.87 & 67.97 & 62.23 & 52.50 & 36.37&60.81 \\
& JoAPR \cite{Guo2024JoAPR}& 72.83 & 71.17 & 70.37 & 67.63 & 65.30 & 57.67 &67.50& 72.07 & 69.80 & 64.10 & 59.17 & 56.07 & 47.46&61.45 \\
& NLPrompt \cite{Pan2025NLPrompt}  & 74.83 & 73.40 & 72.83 & 70.33 & 68.10 & 60.53 &70.00& 74.90 & 73.53 & 71.03 & 65.97 & \textbf{58.97} & \textbf{49.27}&65.61 \\
& \cellcolor{blue!10}VisPrompt & \cellcolor{blue!10}\textbf{78.72} & \cellcolor{blue!10}\textbf{78.88} & \cellcolor{blue!10}\textbf{76.98} & \cellcolor{blue!10}\textbf{74.99} & \cellcolor{blue!10}\textbf{70.34} & \cellcolor{blue!10}\textbf{65.40} & \cellcolor{blue!10}\textbf{74.22} & \cellcolor{blue!10}\textbf{78.80} & \cellcolor{blue!10}\textbf{77.43} & \cellcolor{blue!10}\textbf{72.93} & \cellcolor{blue!10}\textbf{67.43} & \cellcolor{blue!10}56.73 & \cellcolor{blue!10}48.96 & \cellcolor{blue!10}\textbf{67.05} \\ \midrule

\multirow{4}{*}{Caltech101 \cite{caltech101}} & CoOp \cite{CoOp} & 86.43 & 81.03 & 76.73 & 70.90 & 61.33 & 46.90 &70.55& 84.93 & 75.23 & 62.87 & 49.43 & 33.57 & 20.33&54.39 \\
& GCE \cite{GCE}& 92.00 & 90.90 & 90.80 & 89.30 & 86.70 & 79.03 &88.12& 91.27 & 91.20 & 89.73 & 85.80 & 78.20 & 62.07&83.05 \\
& JoAPR \cite{Guo2024JoAPR} & 90.30 & 90.45 & 89.90 & 88.27 & 86.93 & 83.93 &88.30& 90.30 & 89.30 & 88.30 & 88.73 & 85.80 & 81.90&87.39 \\
& NLPrompt \cite{Pan2025NLPrompt} & 91.73 & 91.13 & 90.77 & 89.93 & 88.30 & 86.70 &89.76& 91.60 & 91.17 & 90.20 & \textbf{89.27} & 86.17 & \textbf{81.07}&88.25 \\

& \cellcolor{blue!10}\textbf{VisPrompt} & \cellcolor{blue!10}\textbf{93.14} & \cellcolor{blue!10}\textbf{92.49} & \cellcolor{blue!10}\textbf{91.68} & \cellcolor{blue!10}\textbf{91.16} & \cellcolor{blue!10}\textbf{90.18} & \cellcolor{blue!10}\textbf{89.49} & \cellcolor{blue!10}\textbf{91.36} & \cellcolor{blue!10}\textbf{92.62} & \cellcolor{blue!10}\textbf{92.33} & \cellcolor{blue!10}\textbf{90.79} & \cellcolor{blue!10}88.72 & \cellcolor{blue!10}\textbf{89.37} & \cellcolor{blue!10}80.85 & \cellcolor{blue!10}\textbf{89.11} \\ \midrule
\end{tabular}%
}

\end{table*}

\begin{figure*}[htbp]
\centering
\includegraphics[width=0.98\linewidth]{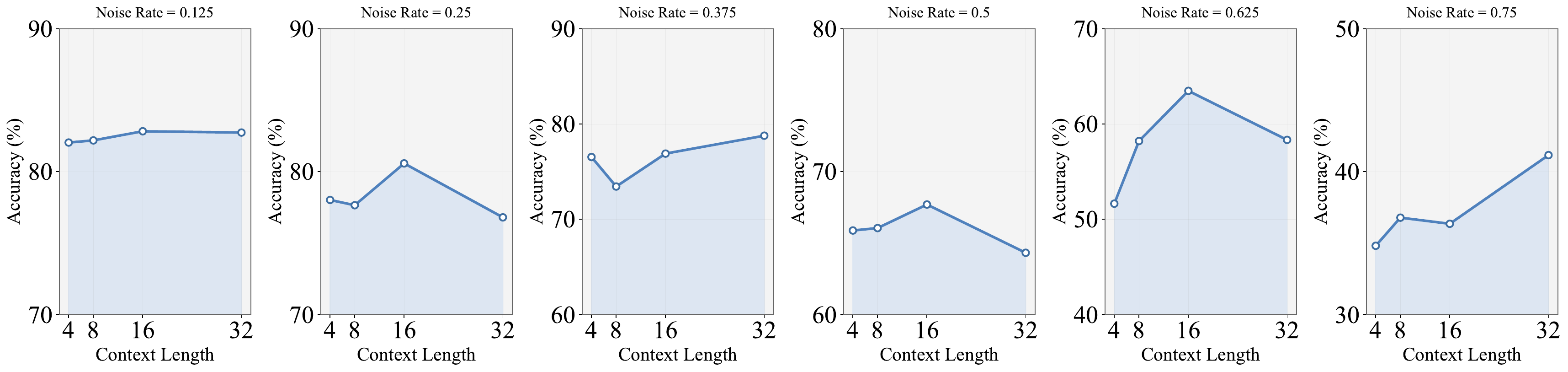} 
\caption{Test accuracy (\%) under different context token lengths}
\label{fig:context_length}
\end{figure*}

\subsection{Main Experiments}


\begin{table}[htbp]
\centering
\caption{Performance on Food101N dataset.}
\setlength{\tabcolsep}{4pt} 
  \resizebox{\columnwidth}{!}{
\begin{tabular}{lccccc}
\toprule 
\textbf{Method} & CoOp \cite{CoOp} & GCE \cite{GCE}& JoAPR \cite{Guo2024JoAPR} & NLPrompt \cite{Pan2025NLPrompt} & VisPrompt \\
\midrule
\textbf{Accuracy} (\%) & 69.50 & 71.32 & 72.57 & 76.46 & \textbf{79.20} \\
\bottomrule
\end{tabular}
}
\label{tab:food101n_perf}
\end{table}

Table \ref{tab:main} summarizes the classification accuracy of VisPrompt and four strong baselines on six benchmarks under both symmetric and asymmetric label noise. Overall, VisPrompt achieves the best performance in all datasets and in almost all noise settings. Averaged over all noise rates, our method outperforms the strong baseline NLPrompt on every dataset for both noise types. The gains are particularly clear on Caltech101. On Flowers102, where the underlying task is relatively easier, VisPrompt still delivers consistent improvements and maintains high accuracy even at very large noise rates. As shown at Table \ref{tab:food101n_perf}, on the real-world Food101N dataset, our framework further delivers a 2.74\% improvement. Another important observation is that VisPrompt degrades much more gracefully as the noise rate increases. Under 75\% asymmetric noise, text-only prompt learning (CoOp) and loss-based robust methods (GCE, JoAPR) suffer large drops on several datasets, while VisPrompt preserves significantly higher accuracy. On EuroSAT, for instance, VisPrompt improves the average accuracy under asymmetric noise from about 53\% with NLPrompt to about 59\%, and at the most challenging 75\% noise rate, it yields more than 15 points gain over loss-based baselines. These results indicate that coupling prompt learning with image–guided FiLM gating provides robustness not only at moderate noise rates but also in the extremely noisy regime, supporting our claim that instance level visual guidance is an effective way to combat label noise.

\subsection{Scalability on Diverse Image Encoders}

We evaluate the scalability of the proposed model on the EuroSAT dataset under symmetric label noise with a noise rate of 0.125. We consider several CLIP visual backbones, including Vision Transformer-Base/16 (ViT-B/16), Vision Transformer-Base/32 (ViT-B/32), ResNet-50x16 (RN50x16), ResNet-50x4 (RN50x4), ResNet-50 (RN50), and ResNet-101 (RN101). As reported in Table \ref{tab:clip_backbone_acc}, the performance remains relatively stable across architectures, suggesting that our approach is not sensitive to the specific backbone choice. More importantly, the consistent gains across both transformer-based and ResNet-based encoders indicate the scalability of the proposed method: it can be seamlessly transferred to backbones with different model capacities and architectural biases without requiring backbone-specific redesign. This result shows that our framework is compatible with a wide spectrum of CLIP variants, from relatively compact models to stronger visual encoders.

\begin{table}[htbp]
  \centering
    \caption{Accuracy (\%) over vision backbones on EuroSAT dataset.}
  \label{tab:clip_backbone_acc}
  
  \resizebox{\columnwidth}{!}{
  \begin{tabular}{c*{6}{c}}
    \toprule 
    \textbf{Backbones} & ViT-B/16 & ViT-B/32 & RN50x16 & RN50x4 & RN50 & RN101 \\
    \midrule
    \textbf{Accuracy (\%)} & 83.62 & 82.83 & 82.78 & 83.67 & 82.83 & 82.69 \\
    \bottomrule
  \end{tabular}
  
  }

\end{table}

\subsection{Ablation Study}

To evaluate the role of image information in guiding prompt generation and the effect of the FiLM \cite{perez2018film} gate in mitigating data drift, we conduct ablation experiments on the EuroSAT dataset, as reported in Fig.~\ref{fig:ablation_combined}. With the CLIP backbone frozen, we compare three variants: w/o vision refers a text-only baseline without visual guidance, w/ vision(no FiLM) refers an image-guided variant that introduces cross-modal residuals but disables the FiLM gate, and the full VisPrompt model that enables both image guidance and the FiLM gate. EuroSAT mainly consists of aerial remote-sensing images, where language priors are relatively weak while visual textures and structural cues are more prominent, which makes it well suited for examining the effectiveness of the vision-to-text pathway and the gating mechanism.

The results show that adding image guidance without gating yields only limited improvements over the text-only counterpart, and its advantage quickly vanishes as the noise rate increases, particularly under asymmetric noise. In contrast, VisPrompt consistently outperforms the other two variants across all noise rates and obtains larger gains in high-noise regimes. Without the gating mechanism, cross-modal residuals integrate image evidence into all context tokens without distinction, so incorrect labels or spurious visual patterns can amplify their interference on prompt learning. The FiLM gate parameterizes the injection strength as a learnable trust coefficient: It increases the degree of integration when visual cues are consistent with textual or class semantics, and reduces it when labels conflict with images or visual evidence is ambiguous. It is worth noting that the parameters introduced by FiLM account for only 0.3\% of the total model parameters, making it lightweight and efficient. Overall, image guidance provides additional information beyond text priors, and the FiLM gate transforms this information into stable generalization under noisy supervision.

\begin{figure}[htbp]
    \centering
    \includegraphics[width=\linewidth]{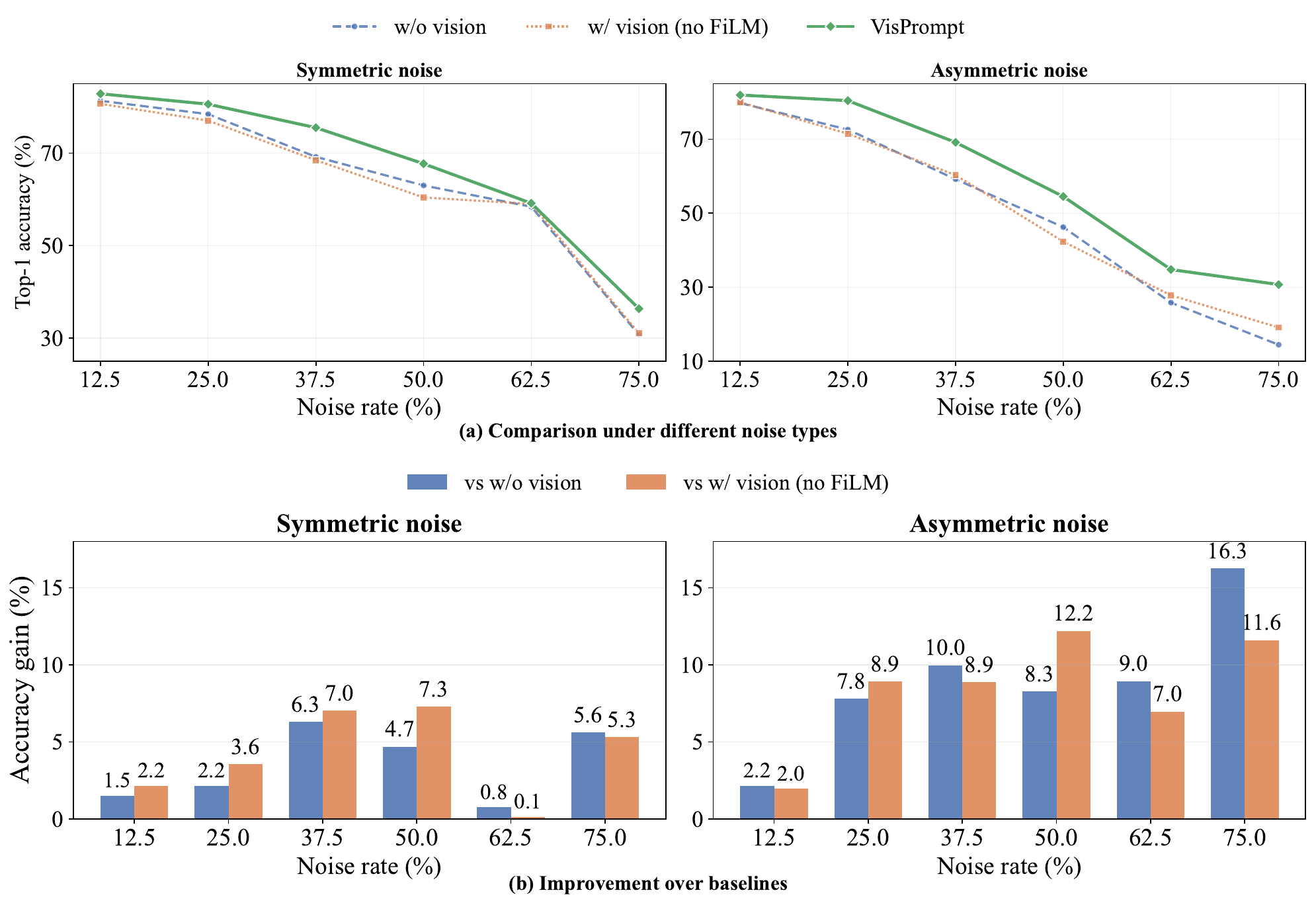}
    \caption{Ablation performance comparison. }
    \label{fig:ablation_combined}
\end{figure}

\subsection{Sensitivity on Different Shots}

To illustrate the impact of the different number of shots within our framework, we conduct experiments on the EuroSAT dataset under the symmetric noise, using training sample counts of [1, 2, 4, 8, 16].
The results are shown in Table \ref{tab:shot_noise}, which reflects that our method can steadily benefit from additional shots and maintain the relatively stable performance as the noise rate increases, indicating strong robustness against label noise and high sample efficiency across a wide range of noise rates.

\begin{table}[htbp]
  \centering
\caption{Accuracy (\%) under symmetric noise at different noise rates and shot settings. }
\label{tab:shot_noise}
  \resizebox{0.85\columnwidth}{!}{
  \begin{tabular}{c*{6}{c}}
    \toprule
    \multirow{2}{*}{Shots} & \multicolumn{6}{c}{Noise Rates} \\
    \cmidrule(lr){2-7}
     & 12.5\% & 25.0\% & 37.5\% & 50.0\% & 62.5\% & 75.0\% \\
    \midrule
    1  & 43.15 & 42.74 & 33.84 & 28.31 & 27.11 & 19.36 \\
    2  & 57.00 & 45.52 & 37.17 & 27.13 & 26.60 & 20.18 \\
    4  & 72.59 & 62.37 & 55.36 & 54.63 & 42.37 & 24.23 \\
    8  & 76.12 & 75.89 & 67.38 & 58.46 & 38.53 & 29.77 \\
    16 & 82.83 & 80.58 & 75.50 & 67.70 & 59.16 & 36.35 \\
    \bottomrule
  \end{tabular}
  }

\end{table}

\subsection{Sensitivity on the Context Length}

To examine the effect of the context token length under varying noise conditions, we conduct experiments on the EuroSAT dataset with context lengths of [4, 8, 16, 32] under symmetric label noise rates ranging from 12.5\% to 75.0\%. The results are shown in Fig.~\ref{fig:context_length}. The performance remains relatively stable across different context lengths, particularly under low and moderate noise rates. Although some fluctuations appear as the noise rate increases, no consistent monotonic relationship is observed between context length and final accuracy. These results indicate that the robustness of our framework does not rely on a narrowly tuned token length and remains effective across a broad range of context configurations.

\section{Discussion}

For limitations, though VisPrompt demonstrates strong robustness to label noise across multiple benchmarks, several limitations still remain. First, the current study is mainly conducted on image classification tasks based on CLIP-style vision-language backbones, which means the generality of the framework beyond this setting has not yet been fully validated. Second, our experiments primarily consider synthetic symmetric and asymmetric label noise, together with only one real-world noisy dataset. As a result, the current evaluation does not yet cover a sufficiently broad range of realistic noise conditions. Third, compared with text-only prompt learning methods, the introduced cross-modal FiLM block brings a moderate amount of additional computational cost. While this overhead is relatively limited in our experiments, it may still become a constraint in highly resource-restricted deployment scenarios.

Regarding future research, there are several directions. First, it would be valuable to extend VisPrompt beyond image classification to more challenging vision tasks, such as object detection, semantic segmentation, and video understanding, in order to assess its effectiveness in broader multimodal learning settings. Second, future work could investigate more diverse and realistic noise patterns, including instance-dependent noise and open-set noise, which would enable a more comprehensive evaluation of the robustness of the proposed framework. Third, further efforts can be devoted to improving the computational efficiency of the cross-modal FiLM block, making the method more suitable for deployment in extremely resource-constrained environments.

\section{Conclusion}
\label{sec:conclusion}
We present VisPrompt, a noise robust prompt learning framework that exploits cross-modal interaction to more fully leverage image information and align image and prompt representations. 
By utilizing visual features as instance-level guidance, VisPrompt enables more accurate prompt generation than relying on text labels alone. The FiLM gating module adaptively regulates visual injection at both token and instance levels, mitigating the adverse effect caused by label noise. Extensive experiments on multiple benchmark datasets with both synthetic and real-world label noise, demonstrate that VisPrompt consistently outperforms existing prompt learning based baselines while remaining parameter efficient.

\bibliographystyle{ACM-Reference-Format}
\bibliography{sample-sigconf-authordraft}

\appendix









\end{document}